\documentclass{article}

     \usepackage[preprint, nonatbib]{neurips_2019}

\usepackage[numbers]{natbib}

\usepackage[utf8]{inputenc} %
\usepackage[T1]{fontenc}    %
\usepackage{hyperref}       %
\usepackage{url}            %
\usepackage{booktabs}       %
\usepackage{amsfonts}       %
\usepackage{amsthm}
\usepackage{amsmath}
\usepackage{array}
\usepackage{nicefrac}       %
\usepackage{microtype}      %
\usepackage{color}
\usepackage{graphicx}
\usepackage{algorithmic}
\usepackage{xr}
\usepackage{enumitem}
\usepackage{float}

\newcommand{\eat}[1]{}

\newtheorem{theorem}{Theorem}
\newtheorem{corollary}[theorem]{Corollary}

\newtheorem{lemma}[theorem]{Lemma}

\theoremstyle{definition}

\newcommand{\defn}{:=}
\DeclareFontFamily{U}{mathx}{\hyphenchar\font45}
\DeclareFontShape{U}{mathx}{m}{n}{
      <5> <6> <7> <8> <9> <10>
      <10.95> <12> <14.4> <17.28> <20.74> <24.88>
      mathx10
      }{}

\DeclareMathOperator*{\argmin}{arg\,min}
\DeclareMathOperator*{\argmax}{arg\,max}

\def\E{{\ensuremath{\mathbb E}}}

\newcommand{\convp}{\ensuremath{\stackrel{p}{\to}}}
\long\def\comment#1{}

\newcommand{\acal}{\ensuremath{\mathcal A}}

\newcommand{\jcal}{\ensuremath{\mathcal J}}

\newcommand{\mcal}{\ensuremath{\mathcal M}}
\newcommand{\ncal}{\ensuremath{\mathcal N}}

\newcommand{\ycal}{\ensuremath{\mathcal Y}}

\title{Manifold Learning via Manifold Deflation}

\author{%
  Daniel Ting \\
  Tableau Research \\
  \texttt{dting@tableau.com} \\
   \And
  Michael I. Jordan \\
  University of California, Berkeley \\
  \texttt{jordan@cs.berkeley.edu} \\
}

\begin{document}

\maketitle

\begin{abstract}
  Nonlinear dimensionality reduction methods provide a valuable means to visualize and interpret high-dimensional data. However, many popular methods can fail dramatically, even on simple two-dimensional manifolds, due to problems such as vulnerability to noise, repeated eigendirections, holes in convex bodies, and boundary bias.
  We derive an embedding method for Riemannian manifolds that iteratively uses single-coordinate estimates to eliminate dimensions from an underlying differential operator, thus ``deflating'' it. These differential operators have been shown to characterize any local, spectral dimensionality reduction method.
  The key to our method is a novel, incremental tangent space estimator that incorporates global structure as coordinates are added. We prove its consistency when the coordinates converge to true coordinates. Empirically, we show our algorithm recovers novel and interesting embeddings on real-world and synthetic datasets.

\end{abstract}

\section{Introduction}
Manifold learning or nonlinear dimensionality reduction (NLDR) methods play an important role in understanding high-dimensional data across a wide range of disciplines such as computer science \cite{tang2015line}, medicine and biology \cite{becht2019dimensionality, esteva2017dermatologist}, and chemistry \cite{dsilva2013nonlinear}. Understanding the manifold structure underlying a model or the raw data can have significant implications for explanatory aspects of machine learning. For example, it can engender trust in models diagnosing cancer \cite{esteva2017dermatologist} or aid scientific discovery by helping biologists identify novel cell types not only via clustering of RNA sequence data, but also by biologically meaningful properties \cite{rostom2017interpretscRNA, haque2017practical}.  Unfortunately, despite two decades of effort, manifold learning is still not at the point that it can be used systematically in applications.  For many of the most popular methods, there is no theory that provides general conditions under which they yield suitable embeddings, and indeed there are known counterexamples.  Of particular concern, extant theory often requires a strong assumption that there is no noise, and the methods indeed often fail in the presence of noise.  These methods can fail dramatically with even small perturbations to the data, such as a modest rescaling or a modest amount of additive noise~\cite{goldberg2008manifold}.

In the current paper, we provide a new framework, and accompanying theory, that mitigates several of these problems.  The framework comprises two algorithmic components.  The primary component is \emph{Manifold Deflation}, a technique that iteratively reduces the dimension of the underlying manifold learning problem as coordinates are added. Each additional coordinate adds a constraint which ensures the next coordinate yields appropriate local geometry and identifies meaningful directions in the tangent space.  The second component is \emph{Vector Field Inversion}, which removes boundary bias by casting a coordinate as a solution to a regression problem that preserves local geometry rather than an eigenproblem that does not preserve geometric properties.  Working together, these two components provide a way to recover novel embeddings and solve or mitigate a number of failures of manifold learning due to scale, noise, holes, and boundaries.

Our framework is best viewed as a modification rather than a wholesale replacement of existing methods.  In particular, the framework requires the specification of a base NLDR method, which can be any of a wide class of local, spectral methods.  This includes some of the most popular NLDR methods, such as Locally Linear Embedding (LLE) \cite{roweis00LLE}, Laplacian Eigenmaps (LE) \cite{belkin2003laplacian}, Diffusion maps \cite{coifman2006diffusion}, Local tangent space alignment (LTSA) \cite{zhang2004principal}, and Hessian LLE (HLLE) \cite{HessEig}.  Used alone, these methods can and do fail, but, as we show, many of the failures can be attributed to a failure to ensure that output embedding coordinates jointly yield suitable local geometries, despite using statistics about local neighborhoods as inputs to generate the embedding.  Our framework provably avoids these failures.

Both Manifold Deflation and Vector Field Inversion are based on a recent dual characterization of local spectral methods, first in terms of a differential operator and a set of boundary conditions and second as a linear smoother \cite{tingnldr}. The original differential operator can be expressed in local coordinate systems as derivatives on $m$ variables where $m$ is the intrinsic dimension of the manifold. Manifold deflation removes one variable as each coordinate is added, thus ``deflating'' the problem into an $(m-1)$-dimensional problem.

Performing this deflation operation requires estimating fundamental objects for a manifold: the tangent vector field induced by a coordinate function and the tangent space at every point.  These are  inferential problems, and we accordingly provide statistical methodology for solving them, by showing how the tangent vector field can be represented as a linear smoother and estimated using local linear regression. We also provide statistical theory, establishing the consistency of this estimator.

\subsection{Related work}
Our methods augment any local, spectral NLDR method \cite{tingnldr}. As we have noted, many of the most popular NLDR methods fall in this class \cite{roweis00LLE, coifman2006diffusion, belkin2003laplacian, zhang2004principal, HessEig}, but there are other methods that do not.  These include Isomap \cite{bernstein2000graph}, t-SNE \cite{maaten2008tsne}, Maximum Variance Unfolding (MVU) \cite{weinberger2006unsupervised}, and UMAP \cite{mcinnes2018umap}.
 
We draw inspiration from work by \cite{goldberg2008manifold} and \cite{gerber2007robust}, who isolated a key structural problem underlying manifold learning algorithms: the problem of scale and repeated eigendirections. Recent work has attempted to address these problems by skipping eigenvectors or by ensuring the next coordinate cannot be predicted from other coordinates~\cite{chen2019selecting, blau2017non-redundant, dsilva2018parsimonious}.   Alternatively, \cite{gerber2007robust} proposed mapping an $m$-dimensional manifold to a lower dimensional manifold and recomputing the Laplacian.

There has been little work addressing boundary bias problems and non-geodesically convex manifolds with holes. Theoretical analyses of methods such as Isomap and MVU \cite{arias2013convergence} assume convexity. While HLLE \cite{HessEig} and the asymptotically equivalent LTSA \cite{zhang2017equivalence, tingnldr} avoid the convexity assumption and boundary bias when there is no noise, their behavior degenerates when there is noise (see Figure \ref{fig:scurve}). Work on handling noise off the manifold in NLDR methods focus on denoising the manifold in the ambient space \cite{hein2007manifold, fefferman2018fitting} or treating the problem as one of outlier detection \cite{chen2006robust}.

\section{Preliminaries}
\label{sec:preliminaries}
Consider an $m$-dimensional compact Riemannian manifold $\mcal$  with smooth boundary smoothly embedded in a high-dimensional ambient space $\mathbb{R}^d$ with $d \geq m$. All of the manifolds we consider in this paper will be assumed to satisfy these properties.
Given $n$ samples $y^{(1)}, \ldots, y^{(n)}$ from $\mcal$, potentially with noise in $\mathbb{R}^d$, we wish to find a good low-dimensional representation of $\mcal$. 
This representation is a mapping $\phi$ from $\mcal$ into $\mathbb{R}^{d'}$ that defines a set of global coordinate functions, $\phi_1(p), \ldots \phi_{d'}(p)$, where $p \in \mcal$. 

The main relevant concepts from differential geometry are the relationships among tangent vectors, differential operators, coordinate charts, and geodesic curves passing through a point. We briefly review these relationships.
Associated with each point $p\in \mcal$ is a tangent vector space $T_p(\mcal)$. Each unit tangent vector $v \in T_p(\mcal)$ is associated with a unique one-dimensional geodesic curve with unit geodesic length passing through $p$ in the direction $v$. In a local neighborhood, a geodesic curve is the unique shortest path connecting two points $p, q$. The point $q$ can be mapped to the vector $v$ tangent to the geodesic curve at $p$ with magnitude $\|v\|$ equal to the geodesic length. Expressing $v$ in a basis of $T_p(\mcal)$ gives a normal coordinate chart for points in a neighborhood of $p$.
 Finally, a tangent vector $v$ naturally defines a differential operator by the partial derivative $f(p) \mapsto \partial f(p) / \partial v$ in the normal coordinate chart. This defines a bijective relationship between tangent vectors and partial derivatives.

Given tangent vectors $v_p$ for all $p \in \mcal$, a vector field $V$ can thus be defined as the mapping $V : C^\infty(\mcal) \to C^\infty(\mcal)$ such that $V f(p) = \partial f(p) / \partial v_p$. Since derivatives are linear operators, finite approximations for $V$ can be represented as a matrix. Our statistical methods are based on estimating and exploiting this matrix. Figure \ref{fig:diff geom} illustrates these relationships.

\begin{figure}
    \hspace{-0.6cm} 
    \includegraphics[width=6in]{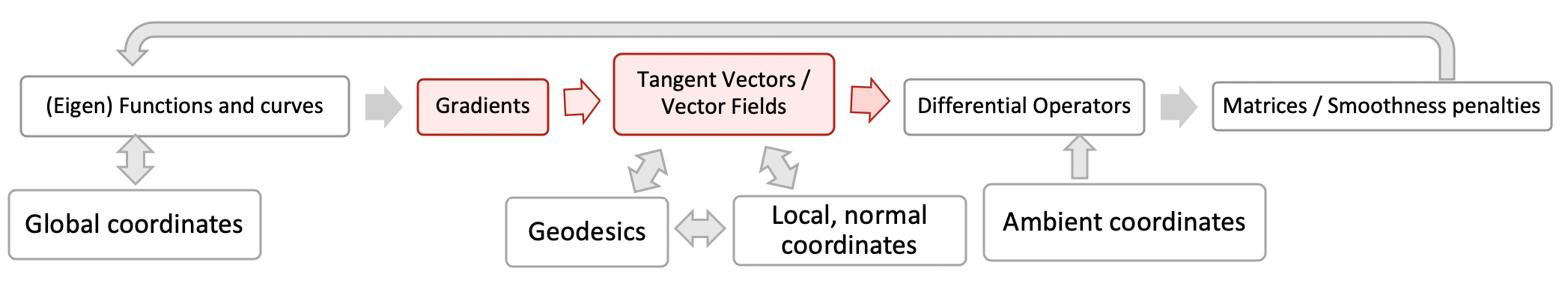}
    \caption{Relationships among differential geometry concepts. Pink denotes the areas our methods uniquely contribute to. These close the loop and allow for iterative refinement of coordinates.}
    \label{fig:diff geom}
\end{figure}

\subsection{Problems}
Given samples from an $m$-dimensional manifold $\mcal$, potentially corrupted with high-dimensional noise to form $\mcal + \mathcal{E}$, our goal is to generate an embedding with
coordinate functions $\phi_i : \mcal + \mathcal{E} \to \mcal' \subset \mathbb{R}^m$ that preserves tangent spaces in the following ways. 
The tangent vectors $\nabla \phi_i(p)$ form an orthogonal basis of $T_p(\mcal)$,
and furthermore the basis is orthonormal. If these conditions are perfectly satisfied at all points $p \in \mcal$, then the embedding is an \emph{isometric embedding}. If only orthogonality is satisfied for all points, then one obtains an \emph{immersion}. As we will see, the latter weakening addresses the practical problem of how to generate good embeddings when local, spectral methods fail and permits discovery of novel, interesting embeddings that reveal the manifold's structure even when an isometric embedding is impossible to attain.

Situations under which local, spectral methods fail include simple problems of scale, non-convexity, and noise. These are illustrated in Figure \ref{fig:scurve}. 
If a manifold has a large aspect ratio---that is, it has underlying coordinates where the range of one coordinate is much longer than the others---multiple eigenvectors may encode the same coordinate direction. This problem is called the \emph{repeated eigendirections problem}~\cite{goldberg2008manifold, gerber2007robust}. Even on a simple two-dimensional s-curve, this problem leads to failures for Laplacian Eigenmaps.

Embeddings can also be changed dramatically by boundaries and holes. 
Some methods, such as LTSA, require estimating local tangent spaces and knowing the intrinsic dimension of the manifold. When a small amount of off-manifold additive noise is added, we prove in the supplementary material that it is impossible to estimate the tangent space when neighborhood sizes shrink to zero. Points are nearly uniformly distributed in a small ball so that meaningful tangent space directions cannot be differentiated from noise directions. Figure \ref{fig:scurve} shows even a small amount of noise causes LTSA to behave like the regular Laplacian.

\begin{figure}
    \hspace{-1.5cm}
    \begin{tabular}{c|c}
    \raisebox{1.2cm}[0pt][0pt]{\rotatebox{90}{Scale}} &  \hspace{-0.4cm}\includegraphics[width=6in,height=1.0in]{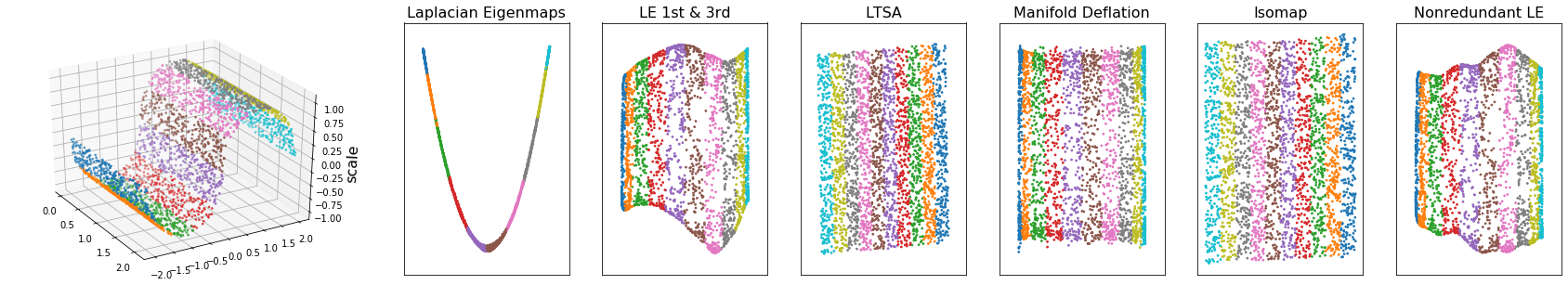} \\
    \raisebox{0.3cm}[0pt][0pt]{\rotatebox{90}{Non-convexity}} & \hspace{-0.4cm}\includegraphics[width=6in,height=1.0in]{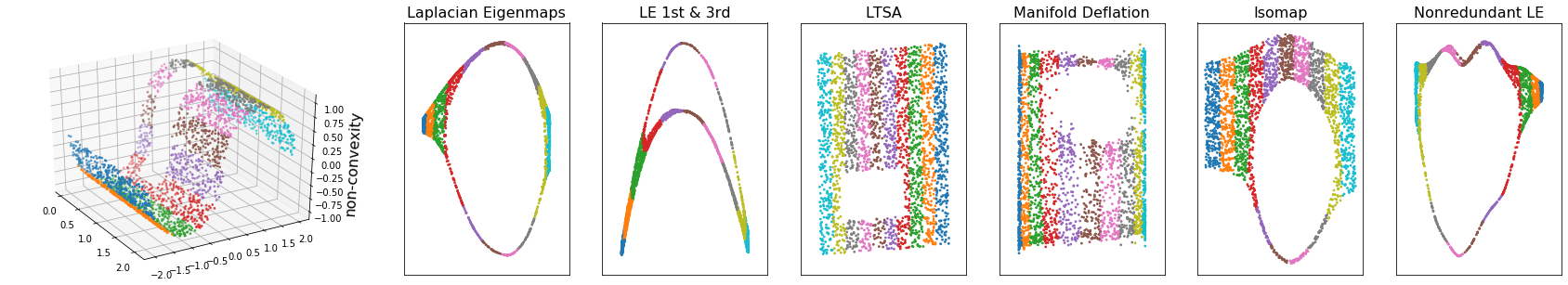}  \\ 
    \raisebox{1.2cm}[0pt][0pt]{\rotatebox{90}{Noise}}  & \hspace{-0.4cm}\includegraphics[width=6in,height=1.0in]{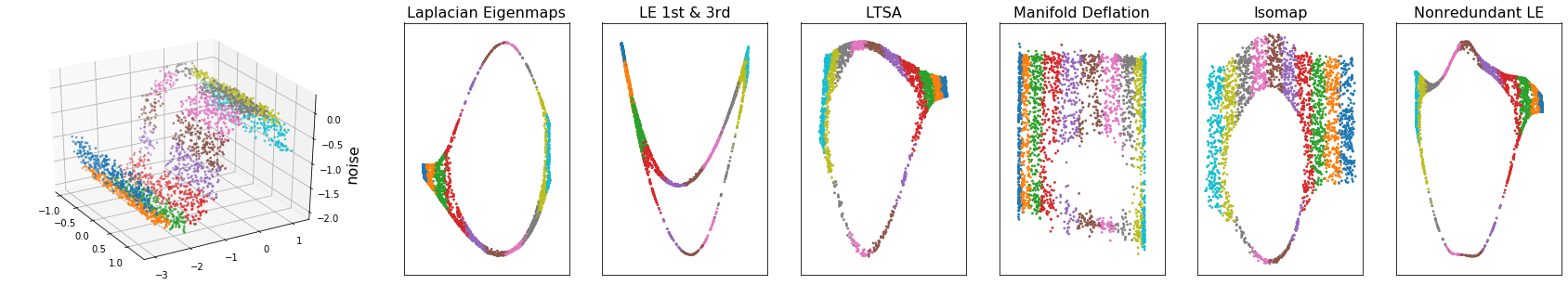} 
    \end{tabular}
    \caption{Scurve}
    \label{fig:scurve}
\end{figure}

\section{Estimators and Embeddings using Differential Operators}
We first briefly review local, spectral methods and their connections to partial differential equations (PDEs) on $m$ variables. We then show manifold deflation corresponds to eliminating one variable at a time from the PDE and describe the geometric connection to orthogonal tangent vectors.

Local, spectral methods consist of three main steps: 1) constructing a neighborhood graph, 2) constructing a sparse matrix $L^{(n)}$ from the local neighborhoods, and 3) solving an eigenproblem using $L^{(n)}$, where the eigenvectors correspond to the embedding. 
Recent work shows that for any local, spectral method, if $L^{(n)}$ converges as neighborhoods shrink in radius, then it must converge to a differential operator and the domain of the operator is specified by a set of boundary conditions~\cite{tingnldr}. For example,  the negative weighted Laplace-Beltrami operator is the limit differential operator for Laplacian Eigenmaps \cite{belkin2003laplacian} and Diffusion Maps \cite{coifman2006diffusion}. Concretely, when points are sampled uniformly from the manifold, the following unweighted Laplace-Beltrami operator is the limit operator, and the domain imposes a Neumann boundary condition:
\begin{align}
 \label{eqn:laplace-beltrami}
    -L^{(n)} f \to \Delta f &= \sum_{i=1}^{m} \frac{\partial^2 f}{\partial x_i^2},\quad \mbox{for}\quad f \in dom(\Delta) = \left\{g : \frac{\partial g}{\partial \nu_p} = 0, \quad\forall p \in \partial \mcal \right\}.
 \end{align} 
 Here, $\nu_p$ is a tangent vector normal to the boundary at $p \in \mcal$. The partial derivatives are taken with respect to local, normal coordinate charts, and convergence is pointwise under the uniform norm. 
Thus, the limit eigenproblem yields estimated coordinates $\hat{\phi} \in dom(\Delta)$ that satisfy the PDE on $m$ variables, $\sum_{i=1}^{m} \frac{\partial^2 \hat{\phi}}{\partial x_i^2} = \lambda \phi$, for some $\lambda$ in a neighborhood with normal coordinates denoted by $\mathbf{x}$. %

\subsection{Manifold deflation}
\label{sec:manifold deflation}
The main idea behind manifold deflation is to eliminate one coordinate from the differential operator. The solutions of the modified differential equations are then functions of the remaining coordinates, thus ``deflating'' the problem from one on $m$ dimensions to one on $m-1$ dimensions.

 We do this by restricting the domain of the operator. Given a coordinate function $\phi_k(p)$, then, with a slight abuse of notation, we impose the constraint that $\partial f(p) / \partial \phi_k(p) = 0$. In Euclidean space this encodes the fact that in Cartesian coordinates the coordinate axes are perpendicular to each other. Correspondingly, the gradients of the coordinate functions are orthogonal, and the directional derivatives satisfy $\partial \phi_j(p) / \partial \phi_k = 0$ whenever $j \neq k$.
 Imposing this constraint forces any term in the differential operator containing the coordinate $\phi_k$ to be zero, so it can be eliminated. 
 
 To formally define the constraint, note that the gradients $\nabla \phi_k(p)$ generate a collection of tangent vectors on $\mcal$. From Section \ref{sec:preliminaries}, each tangent vector corresponds to a directional derivative in a local normal coordinate system. This defines a vector field $\mathbf{V}_k$ mapping a function to its directional derivative at each point.  Thus, the constraint can be expressed as $\mathbf{V}_k f = 0$. 
 We can now formalize our notation, by defining
 $\partial f(p) / \partial \phi_k \defn \|\nabla \phi_k(p)\|^{-2} V_k(p) f$
 as a rescaling of $V_k$.
 
Given a local, spectral method that converges to a positive semidefinite operator $L^\infty$, manifold deflation iteratively solves the following:
 \begin{align}
     \phi_{k+1} &= \argmin_{\substack{f \in dom(L^\infty) \cap \Omega_k \\ 
                                    \langle f, 1  \rangle = 0, \, 
                                    \| f \| = 1}}
                        f^T L^\infty f \quad \mbox{where }\quad 
    \Omega_i = \left\{f \in C^{\infty}(\mcal): V_j f = 0\quad \forall j \leq k\right\}.
\end{align}
 
We note the similarity to matrix deflation methods in numerical linear algebra. The only difference is that matrix deflation imposes an orthogonality constraint $\Omega_k = (\mathrm{span} \{\phi_j : j \leq k\})^\perp$  instead of the vector field constraint. The former are constraints in a function space while the latter are constraints in local neighborhoods and have a geometric interpretation.
The hard constraints on the derivative can be expressed as soft constraints as follows: 
\begin{align}
\label{eqn:soft constraints}
        \phi_{k+1}^{(\lambda)} &= \argmin_{\substack{f \in dom(L^\infty) \cap \Omega_k \\ 
                                    \langle f, 1  \rangle = 0, \, 
                                    \| f \| = 1}}
                        \langle f, L^\infty f\rangle + \lambda \sum_{j \leq i}\| V_j f \|_{L_2(\nu)}^2,
\end{align}
for any measure $\nu$ on $\mcal$ with density bounded away from 0 and $\infty$.
This yields the algorithm:
\begin{enumerate}[noitemsep]
    \item Take as input $L$ from any local, spectral method and a regularization parameter $\lambda$. 
    \item Compute the bottom, non-constant eigenvector $\phi_k$ of $L + \lambda \sum_{j=1}^{k-1} \hat{V}_j^T \hat{V}_j$.
    \item Add $\phi_k$ to the list of coordinates and estimate the matrix valued vector field $\hat{V}_k$ from $\phi_k$.
    \item Increment $k$ and repeat from step 2 until $k = dim(\mcal)$.
\end{enumerate}
    
\subsection{Vector field and tangent space estimation}
The main challenge in performing manifold deflation is estimating the vector field $\hat{V}_k$.
We show that this estimation problem can be reduced to a local linear regression, yielding our proposed Tangent Derivation Estimator.

First, we examine derivative estimation when a coordinate system is known, which we can solve by extending techniques from local polynomial regression \cite{fan1996localpolyregression}. Consider a Taylor expansion in a local, normal coordinate system for a ball $B(p_0,h)$. A point $q$ in the neighborhood has coordinates $x_1(q), \ldots, x_m(q)$, so that
 \begin{align}
     f(q) &= f(p_0) + \sum_{i=j}^{m} \frac{\partial f(p_0)}{\partial x_j}x_j(q) + \epsilon_q,
 \end{align}
 where $\epsilon_q = O(\|h\|^2)$.
Given points $\ncal = \{q_1, \ldots q_\ell\}$ in a small ball around $p_0$ on the manifold and a function $f$, the partial derivative with respect to $x_i$ can be estimated by a local linear regression where the $\ell \times (m+1)$ matrix of regressors is given by the coordinates $M_{ai} = x_i(q_a)$ and the intercept $M_{a,m+1} = 1$. Writing $\beta_k = \frac{\partial f(p_0)}{\partial x_k}$, the derivative estimate is the regression coefficient $\hat{\beta}_k(f) = e_{k}^T (M^T M)^{-1} M^T S f_\ncal \approx \frac{\partial f(p_0)}{\partial x_j}$, where $f_\ncal = (f(q_1), \ldots, f(q_\ncal))$.

Extending this derivative estimator to the case of manifold deflation requires computing an estimate when only a subset of $k$ coordinates out of $m$ are known. We show in the supplementary material that the unknown coordinates can be ignored
when the sampled neighborhoods are sampled uniformly from symmetric balls contained in the interior of $\mcal$. In this case, the normal coordinate functions in the neighborhood asymptotically have zero mean and are orthogonal. Since the unknown coordinates' regressors are orthogonal to the $k^{th}$ coordinate's regressor, they can be ignored. This also implies that other known coordinates can be ignored and the derivative can be estimated using a simple linear regression that depends only on the $k^{th}$ coordinate. 

We now define the Tangent Derivation Estimator.
Let $\mathcal{A}$ be the indices for the $a^{th}$ point's neighbors and let $\phi_{k, \mathcal{A}}$ denote the values of the coordinate of interest, $\phi_k$, evaluated at the $\mathcal{A}$ neighbors. Let $\overline{\phi}_{k, \mathcal{A}}$ be the mean of those values.
A simple linear regression in each neighborhood gives that vector field can be estimated with an $n \times n$ matrix $\hat{V}_k$, where
\begin{align}
    \hat{V}_{k, a, \acal} &= (\phi_{k,\acal} - \overline{\phi}_{k,\acal}) / \|\phi_{k,\acal} - \overline{\phi}_{k,\acal}\|^2. \label{eqn:vf approx}
\end{align}
Here $\hat{V}_{k, i, J}$ denotes the entries indexed by $J$ in the $i^{th}$ row. This estimator has the consistency property given in Theorem \ref{thm:derivative consistency}. The derivation of the estimator given above sketches the proof that it converges to the correct limit.

\begin{theorem}[Consistency of vector field estimate]
\label{thm:derivative consistency}
    Let $\ycal^{(n)}$ be a uniform sample from a smooth $m$-dimensional manifold $\mcal \subset \mathbb{R}^d$ and let the sequence
    $(\epsilon_n)$ decrease to zero. 
    Let  $\phi : \mcal \to \mathbb{R}$ be a smooth function such that $\inf_{p\in \mcal} \|\nabla \phi(p)\| > 0$. Given a smooth function $f$ with bounded gradient.
    Let $\hat{D}_{\phi,\epsilon_n}$ be the estimate of $\partial f / \partial \phi$ given by Eq.~\ref{eqn:vf approx} on an $\epsilon_n$-neighborhood graph.
    Then, for some constant $\alpha$ and for any $h > 0$, if $\epsilon_n \to 0$ sufficiently slowly as $n \to \infty$,
    \begin{align}
    \label{eqn:consistency}
        \sup_{\substack{p \in \mcal \\ d(p, \partial \mcal) > h}} \left|\frac{\partial f}{\partial \phi}(p) - \alpha \hat{D}_{\phi,\epsilon_n} (p)\right| \convp 0.
    \end{align}
\end{theorem}

Since for any $\epsilon > 0$, the estimator $\hat{D}_{\phi,\epsilon}$ is a continuous function of $\phi$ under the uniform norm, the continuous mapping theorem extends the theorem to estimated coordinates $\phi^{(n)}$ that converge to  $\phi$.
\begin{corollary}
 Let $\phi^{(n)}$ be a sequence of smooth, real valued functions such that 
 $\phi^{(n)} \to \phi$ under the uniform norm. 
 Then, under the conditions of theorem \ref{thm:derivative consistency}, Eq. \ref{eqn:consistency} holds with 
 $\hat{D}_{\phi,\epsilon_n}$ replaced with $\hat{D}_{\phi_n,\epsilon_n}$.
\end{corollary}

 \subsection{Example: Laplacian Eigenmaps}
 To illustrate manifold deflation, consider Laplacian Eigenmaps on a uniform sample drawn from a 3D strip,
$[0,9\pi] \times [0, 3\pi] \times [0, \pi] \subset \mathbb{R}^3$. Ignoring the non-smoothness at the corners, this method asymptotically generates the Laplace-Beltrami operator given in Eq.~\ref{eqn:laplace-beltrami}, with boundary conditions
 \begin{align}
 \label{eqn:laplace-beltrami boundary}
    \frac{\partial f(p)}{\partial x_i} = 0 \quad \mbox{if } p_i \in \{0,b_i \pi\} \quad 
    \mbox{where $b_1 = 9, b_2 = 3, b_3=1$}.
 \end{align}

The resulting PDE is separable, such that each coordinate can be treated as an ordinary differential equation (ODE), $-\frac{\partial^2 f}{\partial x_i^2} = \lambda f$, with $\frac{\partial f(0)} {\partial x_i} = \frac{\partial f(b_i)}{\partial x_i} = 0$. The eigenfunctions for this ODE are given by $\cos\left( j x_1 / b_i\right)$ 
for $j \in \mathbb{N}^0$. %

The resulting eigenfunctions for the PDE include these eigenfunctions and their products. The first three nonzero eigenvalues and their corresponding eigenfunctions are given by

\begin{tabular}{lllll}
Eigenvalue & $1/9^2$ & $(2/9)^2$ & $1/3^2$ & $1/3^2 + 1/9^2$  \\ 
Eigenfunctions & $\cos\left( \frac{x_1}{9} \right)$ & $\cos\left( \frac{2x_1}{9} \right)$ & $\cos\left( \frac{x_1}{3} \right), \, \cos\left( \frac{x_2}{3} \right)$ &
$\cos\left( \frac{x_1}{9} \right) \cos\left( \frac{x_2}{3} \right)$ 
\end{tabular}

Thus, the Laplacian Eigenmaps method chooses a function of only $x_1$ for the first three coordinates. A further computation shows that the first 17 eigenfunctions will be solely functions of $x_1$ and $x_2$, so that $x_3$ is not picked up.
Note that the second time $x_2$ appears is in the cross-term
$\cos(x_1/9) \cos(x_2/3)$. This has an eigenvalue that is close to the eigenvalue of $\cos(x_2/3)$ which can cause commonly observed mixing of the eigenvectors when points are sampled randomly. In this case, methods using a predictability heuristic~\cite{blau2017non-redundant, dsilva2018parsimonious} or eigenvector skipping \cite{chen2019selecting} may fail to appropriately distinguish the eigenvectors. We are unaware of work that addresses this form of repeated eigendirections problem. 
On the other hand, manifold deflation eliminates all other eigenfunctions containing $x_1$, including the cross-term $\cos(x_1/9) \cos(x_2/3)$. If the  coordinate $x_3$ with the shortest range is considered a noise dimension, manifold deflation yields a two-dimensional embedding that ignores this noise.

\subsection{Boundary bias and the Vector Field Inversion embedding}
The estimated vector fields can also be used outside of the manifold deflation procedure. We demonstrate one use in removing boundary bias.
Methods that impose a Neumann boundary condition, such as Laplacian Eigenmaps, cannot recover an isometric embedding due to boundary bias~\cite{tingnldr}. 
Trivially, every coordinate function in an isometric embedding must have a gradient with unit norm everywhere. However, the Neumann boundary condition implies the gradient must be zero at some point on the boundary. 

Suppose there exist global coordinate functions $\phi_1, \ldots, \phi_m$ that generate an isometric embedding of the manifold $\mcal$.
The vector field $V_k := \partial / \partial \phi_k$ satisfies $V_k \phi_k = 1$ where $1$ denotes the constant function equal to one. 
To obtain a bias-corrected coordinate, assume $\phi_k$ lies in some reproducing kernel Hilbert space (RKHS). Applying kernel ridge regression to the problem $1 = V_k \phi_k + \epsilon$ yields:
\begin{align}
\label{eqn:debiasing}
    \tilde{\phi}_k &= K (K V_k^T V_k K + \alpha I)^{-1} K V_k^T 1,
\end{align}
for some kernel $K$ and regularizer $\alpha$.
We note that naive matrix inversion of the finite matrix approximation to $V_k$ will not yield meaningful results. We posit this is because the continuous, limit operator $V_k$ is not invertible, since the inverse is not bounded.

\vspace{-0.2cm}
\subsection{Space and time complexity}
Manifold deflation consists of two phases: 1) estimating a vector field $V_k$ and penalty matrix $V_k^T V_k$, and 2) estimating the next coordinate. 
The vector field estimator given in Eq.~\ref{eqn:vf approx} takes time $O(|E|)$ as it simply copies and rescales values from the coordinate into a sparse matrix with $|E|$ nonzero entries where $E$ are the edges in the neighborhood graph.
The number of nonzero entries in the penalty is equal to the number of node pairs two hops or less away in the neighborhood graph. Note there is no dependence on the dimension of the original data or on the intrinsic dimension for the steps unique to manifold deflation. Computing the second smallest eigenvector of a positive semi-definite sparse matrix can be done efficiently with iterative algorithms such as LOBPCG \cite{Knyazev2001TowardTO} or the Lanczos method with inverse iterations.

\vspace{-0.1cm}
\section{Experiments}
We demonstrate the properties of manifold deflation and vector field inversion with experiments on three datasets: an S-curve, a sphere, and the Fashion MNIST dataset. 
In each experiment, we use Laplacian Eigenmaps with an unnormalized Laplacian as the base method for manifold deflation and k-nearest neighbors with 15 neighbors to define the adjacency graph. Additional details about the experiment as well as additional results are included in the supplementary material. 

The S-curve is an isometric embedding of a $3 \times 1$ rectangle in $\mathbb{R}^2$ into an S-shaped curve in $\mathbb{R}^3$. Figures \ref{fig:scurve} shows the behavior of NLDR methods when various challenges are incorporated.  When there is a hole or additive noise, only manifold deflation preserves the two-dimensional structure on the S-curve. Other methods fail despite the small magnitude of the noise, which is additive noise drawn uniformly on a cube with half-width $0.1$.  Non-redundant spectral embedding or methods that skip eigenvectors to deal with the repeated eigendirections problem yield worse embeddings and cannot deal with holes. The supplementary material shows that other combinations of eigenvectors still yield poor embeddings so any method that selects a subset of eigenvectors does poorly.

Due to its use of Laplacian Eigenmaps, manifold deflation exhibits boundary bias in figure \ref{fig:debias} when embedding an S-curve with additive noise. Vector field inversion debiases the coordinates near the boundary to yield strips with uniform width. 

Manifold deflation can also recover embeddings that are useful and qualitatively different from other methods. Figure \ref{fig:sphere} shows that it recovers a rescaling of polar coordinates on a sphere. Points on the sphere are selected from a spherical Fibonacci lattice \cite{hannay2004fibonacci} which deterministically generates uniformly spaced points on a sphere. The north-south and east-west axes are  slightly stretched to make the order of eigenvectors deterministic. Thus, if applied to the globe, manifold deflation approximately recovers a rectangular map. By comparison, all other local, spectral methods do not yield meaningful dimensionality reduction and simply return  a linear projection.

Figure \ref{fig:fmnist} shows embeddings of the Fashion MNIST dataset. Laplacian Eigenmaps essentially collapses the global structure into a collection of one dimensional manifolds. Manifold deflation generates an embedding with interesting manifold structure. The figure shows the embedding on different types of shoes. 
For Laplacian Eigenmaps, it is unclear what can learned from the visualization. Manifold deflation, on the other hand, clearly places darker shades to the left and sandals to the right. Shoes with high tops are close to the bottom while those with low profiles are near the top. The supplementary material provide additional plots comparing manifold deflation to Laplacian Eigenmaps.
We note that while methods such as t-SNE and UMAP provide visualizations that are adept at uncovering the cluster structure of the data, manifold deflation yields insight into the continuous manifold structure of the data.

\begin{figure}
\vspace{-1cm}
\hspace{-1cm}
\begin{tabular}{cc}
    \hspace{-0.2cm}\includegraphics[width=3.0in, height=2.9in]{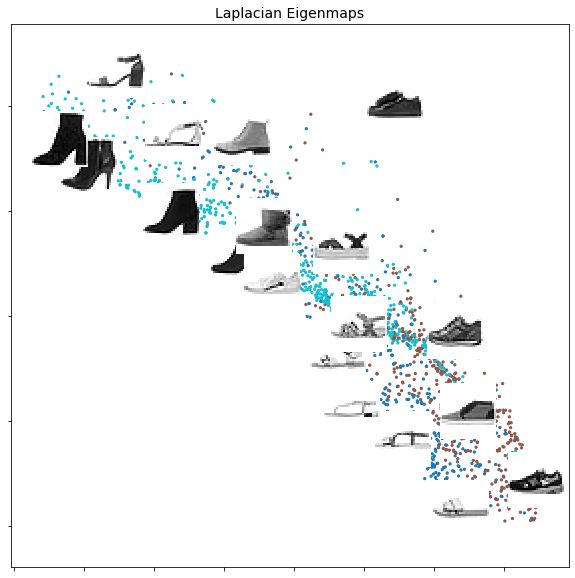} &
    \hspace{-0.3cm}\includegraphics[width=3.3in, height=2.9in]{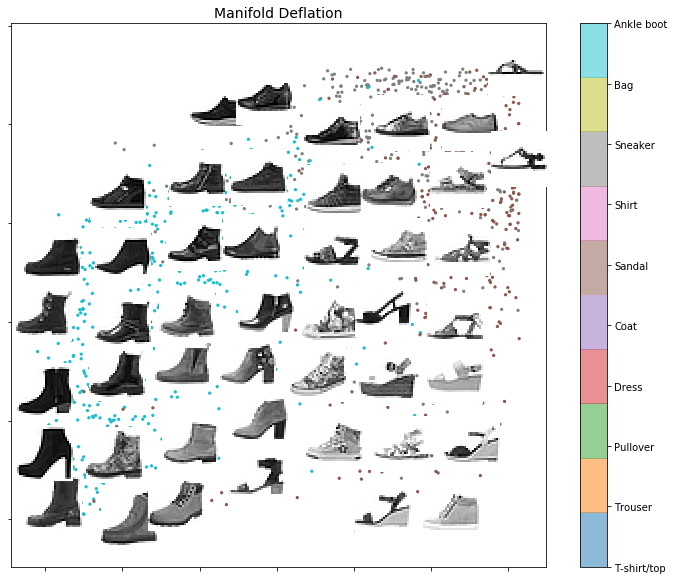} \\
    \hspace{-0.3cm}\includegraphics[width=3.0in,height=1.5in]{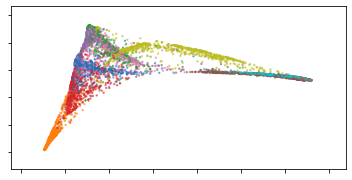} &
    \hspace{-0.3cm}\includegraphics[width=3.3in,height=1.5in]{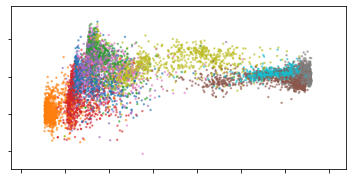}
    \end{tabular}
    \caption{Embedding of the Fashion MNIST dataset colored by category (bottom) and restricted to footwear (top). Manifold deflation avoids collapsing the manifold and reveals interesting variation.}
    \label{fig:fmnist}
\end{figure}

\begin{figure}
    \hspace{-1cm}\includegraphics[width=2.1in, height=1.4in]{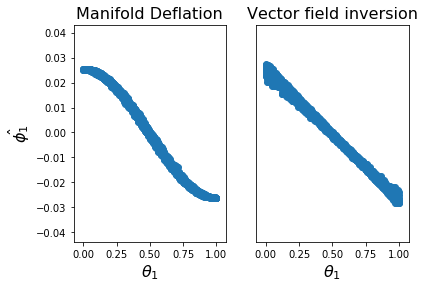}
    \includegraphics[width=4.2in, height=1.4in]{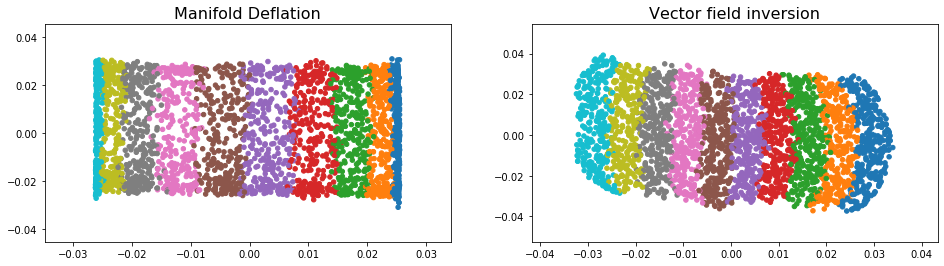}
        \caption{The estimated tangent vector field can be used to debias the coordinate estimates from manifold deflation on an S-curve with noise. The left plots show the first estimated coordinate function versus the true underlying coordinate. Vector field inversion ensures that this relationship is linear; accordingly, the strips in the right plots are of even width.}
    \label{fig:debias}
\end{figure}

\begin{figure}
    \hspace{-0.5cm}\includegraphics[width=6in, height=1.2in]{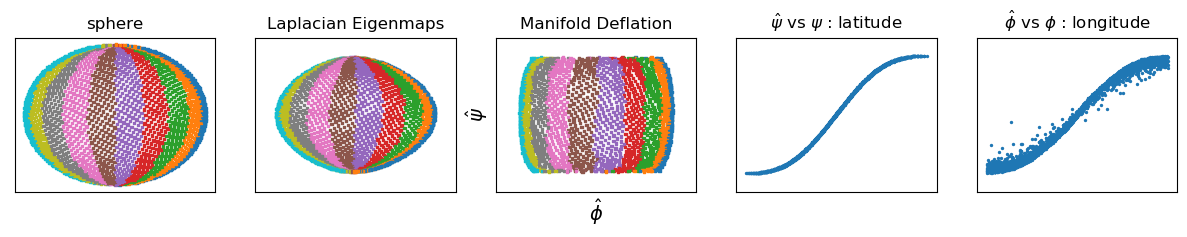}
        \caption{When embedding the surface of a sphere in $\mathbb{R}^2$, manifold deflation recovers polar coordinates (latitude, longitude) for each hemisphere. Laplacian Eigenmaps and other local, spectral methods linearly project into two dimensions. The left plots show slices of longitude values map roughly to vertical slices in the manifold deflation embedding. The right plots show that each recovered coordinate is a monotone transformation of the corresponding polar coordinate.} 
    \label{fig:sphere}
\end{figure}

\vspace{-0.1cm}
\section{Conclusions}
We have proposed a novel framework for manifold learning and nonlinear dimensionality reduction based on Manifold Deflation and Vector Field Inversion. We demonstrate the effectiveness of this framework by successfully recovering a manifold's structure in cases where a local, spectral method will provably fail. 
The common thread among the components of our framework is that they encode information or enforce geometric constraints on coordinates in the coordinate-free language of differential operators that are the foundation of local, spectral NLDR methods. By doing so, they iteratively refine the differential operator of a local, spectral NLDR method as coordinates are sequentially incorporated into an embedding. We crucially develop a key component that enables this expression as differential operators by estimating vector fields and components of the tangent space using the Tangent Derivation Estimator. We have exhibited specific algorithmic applications of this paradigm, but we note that the paradigm may apply much more widely.

\section*{Broader Impact}
We see this work's impact to be a step towards improved visualizations and understanding of high dimensional data and complex machine learning models.
Both of these areas can have significant positive societal impact, but we are unable to anticipate if there are any negative consequences. 
We also believe tangent space estimation may have impact beyond machine learning, such as in mathematics research.

\bibliographystyle{plainnat}
\bibliography{common/ling2}

\begin{thebibliography}{31}
\providecommand{\natexlab}[1]{#1}
\providecommand{\url}[1]{\texttt{#1}}
\expandafter\ifx\csname urlstyle\endcsname\relax
  \providecommand{\doi}[1]{doi: #1}\else
  \providecommand{\doi}{doi: \begingroup \urlstyle{rm}\Url}\fi

\bibitem[Arias-Castro and Pelletier(2013)]{arias2013convergence}
Ery Arias-Castro and Bruno Pelletier.
\newblock On the convergence of maximum variance unfolding.
\newblock \emph{The Journal of Machine Learning Research}, 14\penalty0
  (1):\penalty0 1747--1770, 2013.

\bibitem[Becht et~al.(2019)Becht, McInnes, Healy, Dutertre, Kwok, Ng, Ginhoux,
  and Newell]{becht2019dimensionality}
Etienne Becht, Leland McInnes, John Healy, Charles-Antoine Dutertre,
  Immanuel~WH Kwok, Lai~Guan Ng, Florent Ginhoux, and Evan~W Newell.
\newblock Dimensionality reduction for visualizing single-cell data using umap.
\newblock \emph{Nature Biotechnology}, 37\penalty0 (1):\penalty0 38, 2019.

\bibitem[Belkin and Niyogi(2003)]{belkin2003laplacian}
M.~Belkin and P.~Niyogi.
\newblock {Laplacian eigenmaps for dimensionality reduction and data
  representation}.
\newblock \emph{Neural Computation}, 15\penalty0 (6):\penalty0 1373--1396,
  2003.

\bibitem[Bernstein et~al.(2000)Bernstein, De~Silva, Langford, and
  Tenenbaum]{bernstein2000graph}
Mira Bernstein, Vin De~Silva, John~C Langford, and Joshua~B Tenenbaum.
\newblock Graph approximations to geodesics on embedded manifolds.
\newblock Technical report, Stanford University, 2000.

\bibitem[Blau and Michaeli(2017)]{blau2017non-redundant}
Yochai Blau and Tomer Michaeli.
\newblock Non-redundant spectral dimensionality reduction.
\newblock In \emph{Joint European Conference on Machine Learning and Knowledge
  Discovery in Databases}, pages 256--271. Springer, 2017.

\bibitem[Chen et~al.(2006)Chen, Jiang, and Yoshihira]{chen2006robust}
Haifeng Chen, Guofei Jiang, and Kenji Yoshihira.
\newblock Robust nonlinear dimensionality reduction for manifold learning.
\newblock In \emph{18th International Conference on Pattern Recognition
  (ICPR'06)}, volume~2, pages 447--450. IEEE, 2006.

\bibitem[Chen and Meila(2019)]{chen2019selecting}
Yu-Chia Chen and Marina Meila.
\newblock Selecting the independent coordinates of manifolds with large aspect
  ratios.
\newblock In \emph{Advances in Neural Information Processing Systems}, pages
  1086--1095, 2019.

\bibitem[Coifman and Lafon(2006)]{coifman2006diffusion}
R.~Coifman and S.~Lafon.
\newblock Diffusion maps.
\newblock \emph{Applied and Computational Harmonic Analysis}, 21\penalty0
  (1):\penalty0 5--30, 2006.

\bibitem[Donoho and Grimes(2003)]{HessEig}
D.~L. Donoho and C.~Grimes.
\newblock {Hessian eigenmaps: Locally linear embedding techniques for
  high-dimensional data}.
\newblock \emph{Proceedings of the National Academy of Sciences}, 100\penalty0
  (10):\penalty0 5591, 2003.

\bibitem[Dsilva et~al.(2013)Dsilva, Talmon, Rabin, Coifman, and
  Kevrekidis]{dsilva2013nonlinear}
Carmeline~J Dsilva, Ronen Talmon, Neta Rabin, Ronald~R Coifman, and Ioannis~G
  Kevrekidis.
\newblock Nonlinear intrinsic variables and state reconstruction in multiscale
  simulations.
\newblock \emph{The Journal of Chemical Physics}, 139\penalty0 (18):\penalty0
  11B608\_1, 2013.

\bibitem[Dsilva et~al.(2018)Dsilva, Talmon, Coifman, and
  Kevrekidis]{dsilva2018parsimonious}
Carmeline~J Dsilva, Ronen Talmon, Ronald~R Coifman, and Ioannis~G Kevrekidis.
\newblock Parsimonious representation of nonlinear dynamical systems through
  manifold learning: A chemotaxis case study.
\newblock \emph{Applied and Computational Harmonic Analysis}, 44\penalty0
  (3):\penalty0 759--773, 2018.

\bibitem[Esteva et~al.(2017)Esteva, Kuprel, Novoa, Ko, Swetter, Blau, and
  Thrun]{esteva2017dermatologist}
Andre Esteva, Brett Kuprel, Roberto~A Novoa, Justin Ko, Susan~M Swetter,
  Helen~M Blau, and Sebastian Thrun.
\newblock Dermatologist-level classification of skin cancer with deep neural
  networks.
\newblock \emph{Nature}, 542\penalty0 (7639):\penalty0 115--118, 2017.

\bibitem[Fan and Gijbels(1996)]{fan1996localpolyregression}
Jianqing Fan and Irene Gijbels.
\newblock \emph{Local Polynomial Modelling and its Applications}, volume~66.
\newblock CRC Press, 1996.

\bibitem[Fefferman et~al.(2018)Fefferman, Ivanov, Kurylev, Lassas, and
  Narayanan]{fefferman2018fitting}
Charles Fefferman, Sergei Ivanov, Yaroslav Kurylev, Matti Lassas, and Hariharan
  Narayanan.
\newblock Fitting a putative manifold to noisy data.
\newblock In \emph{COLT}, 2018.

\bibitem[Gerber et~al.(2007)Gerber, Tasdizen, and Whitaker]{gerber2007robust}
Samuel Gerber, Tolga Tasdizen, and Ross Whitaker.
\newblock Robust non-linear dimensionality reduction using successive
  1-dimensional {L}aplacian eigenmaps.
\newblock In \emph{ICML}, 2007.

\bibitem[Goldberg et~al.(2008)Goldberg, Zakai, Kushnir, and
  Ritov]{goldberg2008manifold}
Yair Goldberg, Alon Zakai, Dan Kushnir, and Ya’acov Ritov.
\newblock Manifold learning: The price of normalization.
\newblock \emph{Journal of Machine Learning Research}, 9\penalty0
  (Aug):\penalty0 1909--1939, 2008.

\bibitem[Hannay and Nye(2004)]{hannay2004fibonacci}
JH~Hannay and JF~Nye.
\newblock Fibonacci numerical integration on a sphere.
\newblock \emph{Journal of Physics A: Mathematical and General}, 37\penalty0
  (48):\penalty0 11591, 2004.

\bibitem[Haque et~al.(2017)Haque, Engel, Teichmann, and
  L{\"o}nnberg]{haque2017practical}
Ashraful Haque, Jessica Engel, Sarah~A Teichmann, and Tapio L{\"o}nnberg.
\newblock A practical guide to single-cell rna-sequencing for biomedical
  research and clinical applications.
\newblock \emph{Genome Medicine}, 9\penalty0 (1):\penalty0 75, 2017.

\bibitem[Hein and Maier(2007)]{hein2007manifold}
Matthias Hein and Markus Maier.
\newblock Manifold denoising.
\newblock In \emph{Advances in Neural Information Processing Systems}, pages
  561--568, 2007.

\bibitem[Knyazev(2001)]{Knyazev2001TowardTO}
Andrew~V. Knyazev.
\newblock Toward the optimal preconditioned eigensolver: Locally optimal block
  preconditioned conjugate gradient method.
\newblock \emph{SIAM J. Scientific Computing}, 23:\penalty0 517--541, 2001.

\bibitem[Kosorok(2008)]{kosorok2008empiricalprocesses}
Michael~R Kosorok.
\newblock \emph{Introduction to empirical processes and semiparametric
  inference.}
\newblock Springer, 2008.

\bibitem[McInnes et~al.(2018)McInnes, Healy, and Melville]{mcinnes2018umap}
Leland McInnes, John Healy, and James Melville.
\newblock Umap: Uniform manifold approximation and projection for dimension
  reduction.
\newblock \emph{arXiv preprint arXiv:1802.03426}, 2018.

\bibitem[McQueen et~al.(2016)McQueen, Meil{\u{a}}, VanderPlas, and
  Zhang]{mcqueen2016megaman}
James McQueen, Marina Meil{\u{a}}, Jacob VanderPlas, and Zhongyue Zhang.
\newblock Megaman: scalable manifold learning in python.
\newblock \emph{The Journal of Machine Learning Research}, 17\penalty0
  (1):\penalty0 5176--5180, 2016.

\bibitem[Rostom et~al.(2017)Rostom, Svensson, Teichmann, and
  Kar]{rostom2017interpretscRNA}
Raghd Rostom, Valentine Svensson, Sarah~A. Teichmann, and Gozde Kar.
\newblock Computational approaches for interpreting scrna-seq data.
\newblock \emph{FEBS Letters}, 591\penalty0 (15):\penalty0 2213--2225, 2017.
\newblock \doi{10.1002/1873-3468.12684}.

\bibitem[Roweis and Saul(2000)]{roweis00LLE}
S.~T. Roweis and L.~K. Saul.
\newblock {Nonlinear dimensionality reduction by locally linear embedding}.
\newblock \emph{Science}, 290\penalty0 (5500):\penalty0 2323, 2000.

\bibitem[Tang et~al.(2015)Tang, Qu, Wang, Zhang, Yan, and Mei]{tang2015line}
Jian Tang, Meng Qu, Mingzhe Wang, Ming Zhang, Jun Yan, and Qiaozhu Mei.
\newblock Line: Large-scale information network embedding.
\newblock In \emph{Proceedings of the 24th International Conference on World
  Wide Web}, pages 1067--1077, 2015.

\bibitem[Ting and Jordan(2018)]{tingnldr}
Daniel Ting and Michael~I Jordan.
\newblock On nonlinear dimensionality reduction, linear smoothing and
  autoencoding.
\newblock \emph{arXiv preprint arXiv:1803.02432}, 2018.

\bibitem[van~der Maaten and Hinton(2008)]{maaten2008tsne}
L.~van~der Maaten and G.~Hinton.
\newblock Visualizing data using t-{SNE}.
\newblock \emph{Journal of machine learning research}, 9\penalty0
  (Nov):\penalty0 2579--2605, 2008.

\bibitem[Weinberger and Saul(2006)]{weinberger2006unsupervised}
K.~Weinberger and L.~Saul.
\newblock Unsupervised learning of image manifolds by semidefinite programming.
\newblock \emph{International Journal of Computer Vision}, 70\penalty0
  (1):\penalty0 77--90, 2006.

\bibitem[Zhang et~al.(2017)Zhang, Ma, and Tan]{zhang2017equivalence}
Sumin Zhang, Zhengming Ma, and Hengliang Tan.
\newblock On the equivalence of {HLLE} and {LTSA}.
\newblock \emph{IEEE Transactions on Cybernetics}, 48\penalty0 (2):\penalty0
  742--753, 2017.

\bibitem[Zhang and Zha(2004)]{zhang2004principal}
Z.~Zhang and H.~Zha.
\newblock Principal manifolds and nonlinear dimensionality reduction via
  tangent space alignment.
\newblock \emph{SIAM Journal on Scientific Computing}, 26\penalty0
  (1):\penalty0 313--338, 2004.

\end{thebibliography}

\appendix

\section{Additional methods}

\subsection{Further refinements to vector field estimates}  
\label{sec:refinements}
These vector field estimators assume one has accurate approximations to the underlying coordinates. We implement two refinements to handle curvature and distortions of scale in the estimated coordinate functions. To correct for curvature, the vector $\hat{V}_{k, i, J}$ can be projected onto the column space of the neighbors of $Y_i$ centered on $Y_i$ in the ambient space, in other words, onto the columns space of $\tilde{Y}_J := Y_J - 1 Y_i$. This space contains the tangent space as long as the number of neighbors used is greater than the intrinsic dimension of the manifold.
To correct for distortions of scale, we choose a scaling of the rows such that $\|\hat{V}_{k, i, \cdot} (Y_J - 1 Y_i)\| = 1$ for all indices $i$.
In both cases, the corrections arise since a geodesic curve passing through $p$ on the manifold can be approximated with a linear function passing though $p$ in the ambient space and that the derivative of the geodesic curve is equal to 1. This gives the refinements
\begin{align}
    \tilde{V}_{k,i,J} &= \Pi_i \hat{V}_{k, i, J}^T, \qquad
    \tilde{\tilde{V}}_{k,i,J} = \tilde{V}_{k,i,J} / \| \tilde{V}_{k,i,J} Y_J \|,
\end{align}
where $\Pi_i = \tilde{Y}_J (\tilde{Y}_J^T \tilde{Y}_J)^{-1} Y_J^T$ is the projection onto the column span of $\tilde{Y}$.

\subsection{Choosing the regularization parameter}
\label{sec:regularization}
As the coordinates themselves do not contain information about the scale of the Laplacian $L$ or other base operator, we rescale the penalty $V_k^T V_k$ so that the Frobenius norms of $L$ and the penalty are the same. Given this scaling,
we find manifold deflation is relatively insensitive to the choice of regularization parameter, returning the correct structure for a noisy S-curve with hole for a range of regularization parameter spanning over 3 orders of magnitude from $\lambda \leq 1/2$ to $ \geq 500$. 
We generally found that values slightly larger than 1 worked well. We used a penalty of 3 for our synthetic datasets and 2 for our experiments on Fashion MNIST.

\section{Proofs}

\begin{lemma}[Orthogonality of normal coordinate functions]
\label{lem:orthogonal normal coord}
Let $\delta > 0$ be chosen so that 
for any  $p \in \mcal \subset \mathbb{R}^d$ 
and $ \ncal_\delta^p \defn B(p, \delta) \cap \mcal$,
there exists a neighborhood $\jcal_p$ and normal coordinates $x_1^p(q), \ldots, x_m^p(q)$ for points $q \in \jcal_p$ where the closure of the neighborhood $\overline{\ncal}_\delta^p \subset \jcal_p$.
Denote $M_2 = \E (Z_1^2)$ where $\mathbf{Z}$ is a uniform draw from a unit $m$-sphere and $Z_1$ is the first coordinate.
Then 
\begin{align}
    \sup_{p\in\mcal, \mathrm{dist}(p, \partial \mcal) < \delta} \E ( \delta^{-2} x_i^p(Y) x_j^p(Y) | Y \in \ncal_n^p)  &= \left\lbrace \begin{array}{ll} 
    O(\delta^2) & \mbox{if $i \neq j$} \\
    M_2 + O(\delta^2) & \mbox{otherwise}
    \end{array}
    \right. \\
    \sup_{p\in\mcal} \E ( \delta^{-1} x_i^p(Y) | Y \in \ncal_\delta^p) = O(\delta^2) \\
    \sup_{p\in\mcal, dist(p, \partial \mcal) > \delta} P(Y \in \ncal_\delta^p) &= \gamma M_1 \delta^m (1+ O(\delta^2)) \label{eqn:probability ball}
\end{align}
where $Y$ is a uniform random draw from the manifold and $\nu_m$ is the volume of a $m$-sphere in $\mathbb{R}^m$.
\end{lemma}
\begin{proof}

Let $p \in \mcal$ and $q \in \ncal_\delta^p$. We wish to show that distances in the ambient space are sufficiently close to distances in normal coordinates, and that 
the natural measure on the manifold using the natural volume form can be converted to Lebesgue measure with respect to normal coordinates.
Let $u^p(q)$ denote the projection of $q$ onto the tangent plane.
Lemma 6 in \cite{coifman2006diffusion} relates normal coordinates to the projection into the tangent plane and states that,
$x_i^p = u_i^p + Q_{p,3}(u_i^p) + O(\delta^4)$ where $Q_{p,3}$ is a homogeneous polynomial of degree 3 with coefficients depending on the curvature of the manifold at $p$.
Lemma 7 in \cite{coifman2006diffusion} relates distances in the ambient space to the tangent space and states that  
\begin{align}
\|q-p\|^2 &= \|u^p(q)\|^2  + Q_{p,4}(u^p(q)) + Q_{p,5}(u^p(q)) + O(\delta^6) \\
\det\left(\frac{dq}{du^{p}}\right) &= 1 + Q_{p,2}(u^p(q)) + Q_{p,3}(u^p(q)) + O(\delta^4)
\end{align}
where $Q_{p,\rho}$ are homogeneous polynomials of degree $\rho$ that depend on the curvature of the manifold at $p$. The latter shows the natural volume form is well approximated when integrating using Lebesgue measure on tangent space coordinates.

Thus, we also have the following approximations in a neighborhood with a normal coordinate chart,
\begin{align}
\|q-p\|^2 &= \|x^p(q)\|^2  + Q_{p,4}(x^p(q)) + Q_{p,5}(x^p(q)) + O(\delta^6) \\
\det\left(\frac{dq}{x^{p}}\right) &= 1 + Q_{p,2}(x^p(q)) + Q_{p,3}(x^p(q)) + O(\delta^4). \label{eqn:volume form}
\end{align}
These allow us to approximate balls in the ambient space with balls in normal coordinates and to approximate an integral on the manifold with an integral in Euclidean space.

Let $q_0 = \argmax_{q \in \overline{\ncal_\delta^p}} \|\mathbf{x}(q)\|$ be the furthest point in normal coordinates in the closed neighborhood.
Since geodesic distances are at least as long as straight line distances in the ambient space, the ball of radius $\delta$ in normal coordinates must be contained in the neighborhood $\jcal_p$ and $\|\mathbf{x}(q_0)\| \geq \delta$. Furthermore, 
since the furthest point in the tangent space that can be in the neighborhood  trivially has distance $\delta$ from $p$, one has $\|\mathbf{x}(q_0)\| = \delta + O(\delta^3)$.
Thus, the neighborhood $\ncal_\delta^p$ can be approximated by a ball $B_x(p, \delta)$ of radius $\delta$ in normal coordinates. The difference in measure between the two balls is
$\mu(\ncal_\delta^p \backslash B_x(p, \delta)) \leq \gamma (s_m \delta^{m-1}) \eta \delta^{3}$
where $\eta >0$ is some constant, $s_m \delta^{m-1}$ is the surface area of a unit sphere in $m$ dimensions, and $\gamma$ is the uniform density on the manifold. Denote 
$\mcal^0 =\{ p \in \mcal : dist(p, \partial \mcal) > \delta \}$.
Then,
\begin{align}
    \sup_{p\in\mcal^0} &P(Y \in \ncal_\delta^p) \\
    &= \int_{B_x(p,\delta)} \gamma dx^p_1 \cdots dx^p_m (1+O(\delta^2)) + \gamma s_m \eta \delta^{m-1} O(\delta^3) \nonumber \\  
    &= \gamma \nu_m \delta^m(1+ O(\delta^2)) + \gamma s_m \eta \delta^{m-1} O(\delta^3)\nonumber \\
    &= \gamma \nu_m \delta^m(1+ O(\delta^2))
\end{align}
where $\nu_m$ is the volume of a unit ball in $m$-dimensions, and equation \ref{eqn:volume form} is used to convert an expectation using normal coordinates to an expectation with respect to the measure on the manifold.

Likewise, 
\begin{align}
    & \sup_{p\in \mcal^0} \E ( \delta^{-2} x_i^p(Y) x_j^p(Y) 1(Y \in \ncal_\delta^p)) \\ 
    &\quad= 
    \sup_{p\in \mcal^0} \E ( \delta^{-2} x_i^p(Y) x_j^p(Y) 1(Y \in B_x(p, \delta))) + 
    \gamma (1+O(\delta^3)) \gamma s_m O(\delta^{m-1}) \eta \delta^3
    \nonumber  \\
    &\quad= 
    \sup_{p\in \mcal^0} \E ( \delta^{-2} x_i^p(Y) x_j^p(Y) | Y \in B_x(p, \delta)) p(Y \in B_x(p, \delta)) + 
    O(\delta^{m+2}) \nonumber \\
    &\quad= 
    \sup_{p\in \mcal^0} \E ( \delta^{-2} x_i^p(Y) x_j^p(Y) | Y \in B_x(p, \delta)) \gamma \nu_m \delta^m(1+ O(\delta^2)) +
    O(\delta^{m+2}) \nonumber \\
    &\quad= 
    \sup_{p\in \mcal^0} \E ( \delta^{-2} x_i^p(Y) x_j^p(Y) | Y \in B_x(p, \delta)) \gamma \nu_m \delta^m +
    O(\delta^{m+2}) \nonumber \\
    \nonumber \\
    & \sup_{p\in \mcal^0} \E ( \delta^{-1} x_i^p(Y) 1(Y \in \ncal_\delta^p)) \\ 
    &\quad= 
    \sup_{p\in \mcal^0} \E ( \delta^{-1} x_i^p(Y) 1(Y \in B_x(p, \delta)))(1 + O(\delta^2)) + \gamma (1+O(\delta^3)) S_m O(\delta^{m-1}) \eta \delta^3 \nonumber  \\
    &\quad= 
    \sup_{p\in \mcal^0} \E ( \delta^{-1} x_i^p(Y) | Y \in B_x(p, \delta)) \gamma \nu_m \delta^m + O(\delta^{m+2}). \nonumber
\end{align}
The first moment and second moment when $i \neq j$ are $0 +O(\delta^{m+2})$ since the $x_i$ and $x_i x_j$ are odd functions. Likewise, the second moment when $i = j$ is $c + O(\delta^{m+2})$ for some constant $c$ that does not depend on $i$ since $x_i^2$ is even.

\end{proof}

\begin{theorem}[Consistency of vector field estimate]
    Let $\ycal^{(n)}$ be a uniform sample from a smooth $m$-dimensional manifold $\mcal \subset \mathbb{R}^d$ and let the sequence
    $(\epsilon_n)$ decrease to zero. 
    Let  $\phi : \mcal \to \mathbb{R}$ be a smooth function such that $\inf_{p\in \mcal} \|\nabla \phi(p)\| > 0$. Given a smooth function $f$ with bounded gradient.
    Let $\hat{D}_{\phi,\epsilon_n}$ be the estimate of $\partial f / \partial \phi$ given by Eq.~\ref{eqn:vf approx} on an $\epsilon_n$-neighborhood graph.
    Then, for some constant $\alpha$ and for any $\delta > 0$, if $\epsilon_n \to 0$ sufficiently slowly as $n \to \infty$,
    \begin{align}
        \sup_{\substack{p \in \mcal \\ d(p, \partial \mcal) > \delta}} \left|\frac{\partial f}{\partial \phi}(p) - \alpha \hat{D}_{\phi,\epsilon_n} (p)\right| \convp 0.
    \end{align}
\end{theorem}
\begin{proof}
We prove the theorem in two-steps. First, we show a non-stochastic Tangent Derivation Estimator with infinite samples converges to the desired differential operator. Second, we show stochastic convergence to this infinite sample limit.

For clarity of exposition in the non-stochastic portion of the proof, we prove the case for a fixed $p$ and show that it can be easily modified to show joint convergence over all $p$ sufficiently far from the boundary.
Let $p \in \mcal$. Denote $\psi_{p, \epsilon}(q) = (\phi(q) - \phi(p)) 1(\|p - q\| < \epsilon)$
where the norm is the Euclidean norm in $\mathbb{R}^d$.
We wish to find the asymptotic limit of the Tangent Derivation Estimator at $p$. We denote the Tangent Derivation Estimator as 
$D_{\phi, \epsilon_n} = \langle f, \psi_{p, \epsilon_n} \rangle / \| \psi_{p,\epsilon_n}\|^2$.
Consider a Taylor expansion of $f$ and $\psi_{p, \epsilon_n}$ in normal coordinates in the neighborhood $B(p, \epsilon_n) \cap \mcal$ where the ball is in the ambient space $\mathbb{R}^d$. 
\begin{align}
\label{eqn:taylor f}
    f(q) &= f(p) + \nabla f(p) \mathbf{x}(q) + O(\epsilon_n^2) \\
\label{eqn:taylor psi}
\psi_{p,\epsilon_n}(q) &= \nabla \phi(p) \mathbf{x}(q) + O(\epsilon_n^2)
\end{align}
By lemma \ref{lem:orthogonal normal coord},
\begin{align}
    \langle f, \psi_{p,\epsilon_n} \rangle &\defn \E f\psi_{p,\epsilon_n}= 
     \gamma \nu_m \epsilon_n^m (\langle \nabla f(p), \nabla \phi(p)\rangle + O(\epsilon_n^2)) \\
    \|\psi_{p,\epsilon_n} \|^2 &= \gamma \nu_m \epsilon_n^m ( \|\nabla \phi(p)\|^2 + O(\epsilon_n^2))
\end{align}
where $\gamma >0$ does not depend on $\epsilon_n$ or $p$.
Taking $\alpha = 1/\gamma \nu_m$ gives
\begin{align}
\frac{\partial f}{\partial \phi}(p) = \alpha D_{\phi,\epsilon_n} + O(\epsilon_n^2).
\end{align}

To obtain the desired convergence result under all $p$ sufficiently far from the boundary, note that equations \ref{eqn:taylor f} and \ref{eqn:taylor psi} for fixed $p$ have the 
counterparts
\begin{align}
    \sup_{p \in \mcal, \|q-p\| < \epsilon_n} \|f(q) - f(p) - \nabla f(p) \mathbf{x}(q)\| &= O(\epsilon_n^2) \\
    \sup_{p \in \mcal, d(p, \partial \mcal) > \epsilon_n, \|q-p\| < \epsilon_n} \|\psi_{p,\epsilon_n}(q) - \nabla \phi(p) \mathbf{x}(q)\| &= O(\epsilon_n^2)
\end{align}
since the functions $f, \phi$ are smooth and the manifold is compact. 
Following the same steps as before yields the desired convergence result.

This proves convergence of the expectation of the estimate to the true partial derivatives. We now prove convergence in probability of the stochastic processes. We rely primarily on the VC class and BUEI preservation rules given in \cite{kosorok2008empiricalprocesses}.
First, we note that all the expectations used in this proof and of lemma \ref{lem:orthogonal normal coord} are with respect to a continuous uniform measure on the manifold. Thus, any expectation over a closed ball is the same as that over the corresponding open balls. The set of all closed balls in $\mathbb{R}^d$ has finite VC-dimension. The functions $\psi_{p, \epsilon_n}$ are a product of a fixed function and an open ball. The product of a fixed function with VC-class is still a VC-class. Thus, $\{f(\cdot) \psi_{p, \epsilon}(\cdot) \}_{p, \epsilon}$ forms a VC-class. 
Likewise, since $\psi_{p,\epsilon}$ can be split into positive and negative parts
and $x \to x^2$ is monotone on either positive reals or on negative reals,
the class $\{ \psi_{p, \epsilon}(\cdot)^{2} \}_{p, \epsilon}$ has finite VC-dimension.
Since $\mcal$ is compact and $f, \phi$ are smooth then the VC-classes are also bounded and thus have an integrable envelope.
 VC-classes have bounded uniform entropy integral (BUEI) if they have an integrable envelope.
 Furthermore, the class of constant functions $\{1/\epsilon^{2m}\}_{\epsilon \geq r}$
 is trivially BUEI if $\epsilon$ is restricted so that $\epsilon \geq r$ for some $r > 0$.
 Since products of BUEI classes are still BUEI,
 and BUEI classes are P-Glivenko-Cantelli, it follows that 
 $\{f(\cdot) \psi_{p, \epsilon}(\cdot) / \epsilon^{2m} \}_{p, \epsilon \geq r}$ and $\{\psi_{p, \epsilon}^2(\cdot) / \epsilon^{2m} \}_{p, \epsilon \geq r}$
 are P-Glivenko-Cantelli.

Let $\mathbb{P}_n$ denote the empirical distribution of $n$ samples, and write
$\hat{D}^n_{\phi, \epsilon}(p) = \mathbb{P}_n f \psi_{p, \epsilon} / \mathbb{P}_n \psi_{p,\epsilon}^2$. Since the numerator and denominator are both P-Glivenko-Cantelli and the denominator is bounded away from 0 almost surely, by the continuous mapping theorem,
$\sup_{p \in \mcal} |\hat{D}^n_{\phi,\epsilon}(p) - D_{\phi,\epsilon}(p)| \to 0$ almost surely for any $\epsilon > 0$ as $n \to \infty$.

We weaken this to convergence in probability as we take $\epsilon \to 0$.
Since almost sure convergence implies convergence in probability, for any sequence $\tau_t \downarrow 0$, we can find a corresponding sequences $n_t \to \infty$ and $\epsilon_t \downarrow 0$ so that
\begin{align}
    \sup_{p\in \mcal} P(|\hat{D}^n_{\phi, \epsilon_t}(p) - D_{\phi, \epsilon_t}(p)| > \tau_t) < 2^{-t} \quad \forall n > n_t.
\end{align}
Since when $dist(p, \partial \mcal) > h$, $\alpha D_{\phi, \epsilon_t}(p) \to \partial f(p)/ \partial \phi$ as $\epsilon_t \downarrow 0$, it follows that 
    \begin{align}
        \sup_{\substack{p \in \mcal \\ d(p, \partial \mcal) > h}} \left|\frac{\partial f}{\partial \phi}(p) - \alpha \hat{D}_{\phi,\epsilon_n} (p)\right| \convp 0.
    \end{align}

.

\end{proof}

 \subsection{Impossibility of estimating the tangent space given noise}
 Here we show that that when there is off manifold noise, then in a small neighborhood of any point on the manifold, the points in a local neighborhood are asymptotically equivalent to a uniform draw from a ball in the high-dimensional space $\mathbb{R}^d$. Thus, there is no information about the tangent space just from that neighborhood.
 
\begin{lemma}
\label{lem:uniform ball}
    Let $h_n$ be a decreasing sequence converging to 0.
    Let $Y_n = Z + h_n \epsilon$ where $Z$ is a uniform random sample from a manifold $\mcal \subset \mathbb{R}^d$ and $\epsilon$ is independent noise drawn from some distribution $F$ on $\mathbb{R}^d$ with a smooth density $f$. Further assume $f(0) > 0$ and has bounded support.
    For any $p \in Int(\mcal)$ and sequence $\delta_n$ such that $\delta_n / h_n \to 0$, the conditional densities $g_n$ for $X_n = \delta_n^{-1}(Y_n - p)\,  | \,\|Y_n - p\| < \delta_n$ converge with $g_n(x) \to c$ on the unit ball for some constant $c$.  
\end{lemma}
\begin{proof}
Let $\ncal_{\delta}^p = B(p, \delta) \cap \mcal$.
Denote by $\mu$ the uniform measure on the manifold.
The density of $Y_n - p$ is given by 
\begin{align}
    \rho(y) &= \int_{\ncal_{\delta}^p} f( (x-p)/h_n) d\mu(x) \\
    &= \int_{\ncal_{\delta}^p} f(0) + \nabla f(0) (x-p)/h_n  + O(\delta^2/h_n^2) d\mu(x) \\
\end{align}

The first component can be computed by
\begin{align}
\int_{\ncal_{\delta}^p} f(0) d \mu(x) &= f(0) \nu_m \delta^m (1 + O(\delta^2)).
\end{align}

For the second, w.l.o.g. we can assume the coordinates are parameterized so that
the first $m$ coordinates $x_1, \ldots, x_m$ parameterize the tangent space at $p$.
The remaining coordinates can thus be written as $x_{m+i} = q(x_1, \ldots, x_m) + O(\delta^3)$ where
$q$ is a homogeneous polynomial of degree 2. This gives 
\begin{align}
h_n^{-1} \int_{\ncal_{\delta}^p} (x_i-p_i) d \mu(x) &= 
\nu_m \delta^m (0 + O(\delta^3)) (1 +  O(\delta^2)) 
\end{align}
where the $0$ term comes from evaluating the integral on the first $m$ coordinates and the $O(\delta^3)$ term comes from evaluating the integral on degree 2 polynomials on the remaining coordinates.

Dividing by the probability $\gamma \nu_m \delta^m (1+O(\delta^2))$ of being in the neighborhood that is given in equation \ref{eqn:probability ball} yields the desired result.
\end{proof}

\section{Experiments}

\subsection{Details}
For our experiments we considered both synthetic datasets and a real world dataset.
The synthetic datasets consisted of variations on an S-curve and a sphere.
The S-curve was generated by taking random samples from a $3 x 1$ rectangle. The rectangle is isometrically embedded into $\mathbb{R}^3$ by mapping each $1.5 \times 1$ half of the rectangle to half-cylinders and ensuring the mapping is continuous.
By taking the length of the original rectangle to be more than $2 \times$ the width, we ensure that the repeated eigendirections problem is encountered.
We modified the S-curve in two ways to generate two variations of the manifold. We cut out a rectangular hole, and we added a small amount of additive noise. The noise is i.i.d. from a cube with half-width 0.1, so that it is a small fraction of the length and width of the underlying rectangle. 
For the sphere we deterministically construct points using 
For each synthetic datasets we generated 3000 points each. For the S-curve with hole, this number was reduced by roughly 12\% due to cutting out a hole.

We used Fashion MNIST (FMNIST) for the real world dataset. We ran our experiments on 9000 randomly selected points from the 60000 training examples. Each example consists of a $28 \times 28$ pixel greyscale image converted into a vector in $744$ dimensions.

Experiments were run on a Macbook with a 2.2 GHz Intel i7 processor with 6 cores and 16GB of memory. On the FMNIST data, Manifold Deflation took 10 minutes to run to generate 2 coordinates. In comparison, Laplacian Eigenmaps took 3 minutes, and Non-redundant Laplacian Eigenmaps took 25 minutes. For the synthetic data sets with 3000 points and 3 dimensions, all methods combined finished in under 2 mins.

\subsection{Parameters}
The only manifold deflation specific parameters are the regularization parameter and a scaling of the vector field. The regularization parameter was chosen to be 3 for synthetic datasets and 2 for FMNIST as described in section \ref{sec:regularization}. For synthetic datasets, we used the refinements in section \ref{sec:refinements}. For the FMNIST dataset, we simply normalized each row of the estimated matrix $\hat{V_k}$ to have norm $1$.

For all graph constructions, we use k-nearest neighbors with 15 nearest neighbors.
Given a kNN graph, the edges are weighted with a Gaussian kernel with bandwidth $\sigma$. The bandwidth $\sigma$ is taken to be $5 \times $ the average distance from a point to its neighbor. For LTSA, we set the intrinsic dimension to be 2 for the synthetic datasets. We find that if it is set to the ambient dimension 3, then LTSA does not perform any meaningful dimensionality reduction.

\subsection{Code}
Our code is based off of the Megaman package for manifold learning using local, spectral methods \cite{mcqueen2016megaman}. We implemented Manifold Deflation, Tangent Derivation Estimation, and Vector Field Inversion as well as Non-Redundant Laplacian Eigenmaps for comparison. Rather than implement other eigenvector skipping procedures that also address the repeated eigendirections problem, we simply plot additional eigenvectors to show that they all yield bad embeddings.

In all cases, we used dense matrix computations and exact nearest neighborhood computations to avoid distortions coming from other sources. We did not find any qualitative differences when modestly increasing or decreasing the data set sizes. We found in the high dimensional FMNIST example, approximate nearest neighbor computations in megaman generated distortions in Manifold Deflation that were not visible when just using Laplacian Eigenmaps (LE) since LE collapsed the manifold structure. We did not further pursue the issue of scale in this implementation as it is tangential to the main ideas of this paper.

\section{Additional experimental results}

\subsection{Repeated eigendirections}
Figures \ref{fig:all pairs}, \ref{fig:box3d}, \ref{fig:all pairs box3d} illustrate cases where eigenvector skipping methods can never return a good embedding. 
In figures \ref{fig:box3d}, \ref{fig:all pairs box3d}, while Non-redundant Laplacian Eigenmaps returns a reasonable embedding, it distorts the manifold and makes one end larger than the other. 

\begin{figure}[H]
    \hspace{-2cm}
    \includegraphics[width=6.2in]{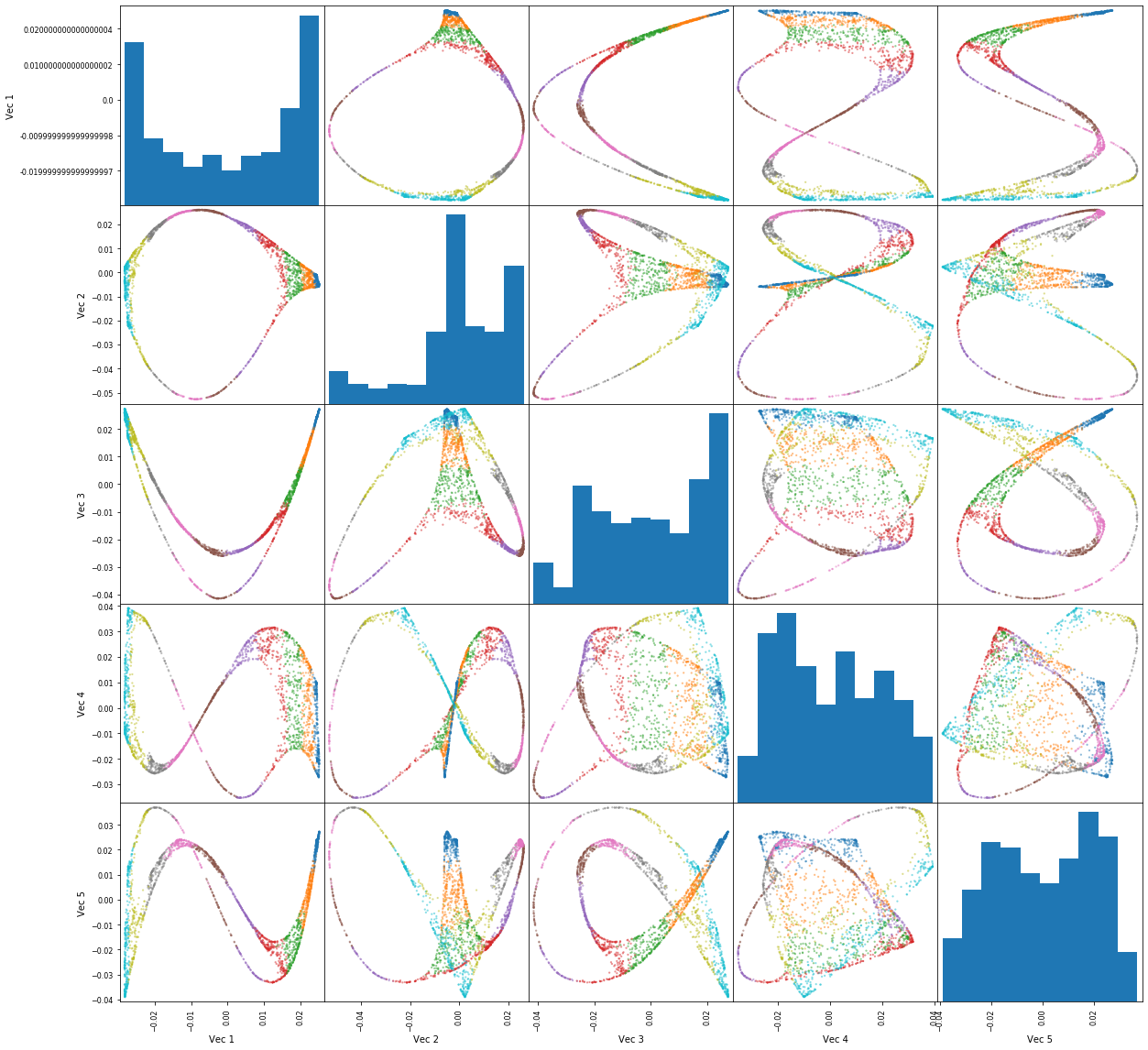}
    \caption{Pair plot of first five nonconstant eigenvectors from Laplacian Eigenmaps when run on an S-curve with hole and noise. No pair of eigenvectors yields a good embedding.}
    \label{fig:all pairs}
\end{figure}

\begin{figure}[H]
    \hspace{-2cm}
    \includegraphics[width=6.2in]{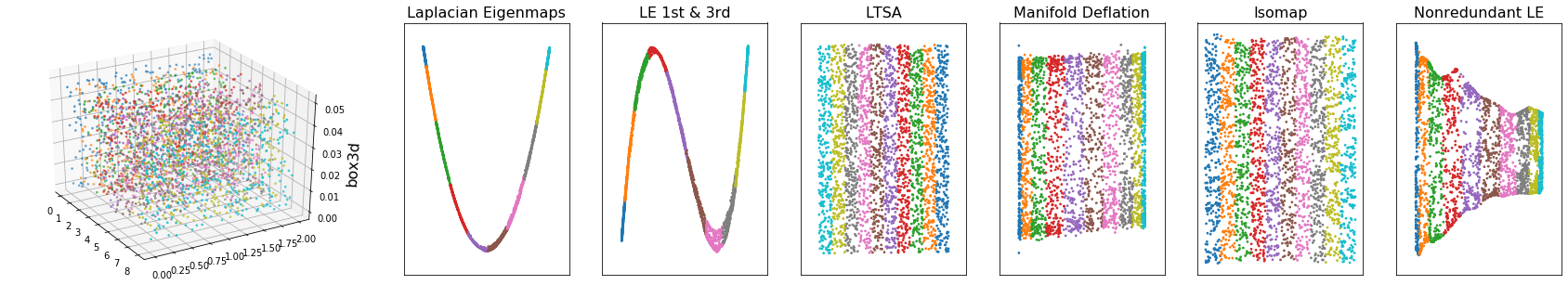}
    \caption{Embeddings for a $10 \times 2 \times 0.1$ box. Local, spectral methods display the repeated eigendirections problem.}
    \label{fig:box3d}
\end{figure}

\begin{figure}[H]
    \hspace{-2cm}
    \includegraphics[width=6.2in]{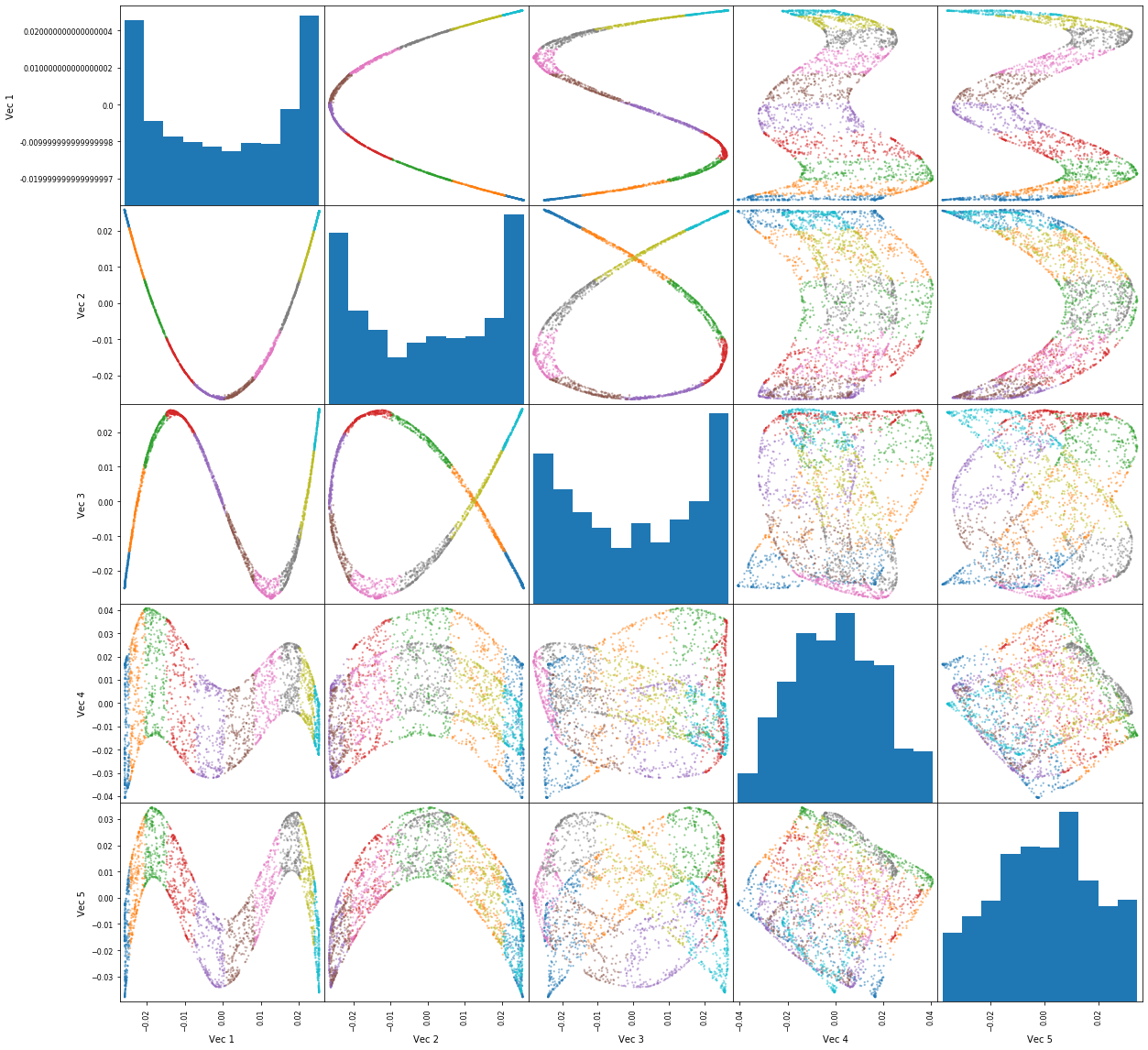}
    \caption{Pair plot of first five nonconstant eigenvectors from Laplacian Eigenmaps when run on a $10 \times 2 \times 0.1$ box. No pair of eigenvectors yields a good embedding.}
    \label{fig:all pairs box3d}
\end{figure}

\subsection{Vector field inversion}
\begin{figure}[H]
    \centering
    \includegraphics[width=3in]{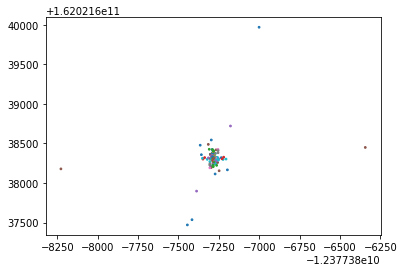}
    \caption{Solutions of $V_1^{-1} 1$ and $V_2^{-1} 1$. Naive inversion of a vector field yields meaningless embeddings.}
    \label{fig:naive VFI}
\end{figure}

\subsection{Additional sphere plots}
\begin{figure}[H]
    \hspace{-2cm}
    \includegraphics[width=6.2in]{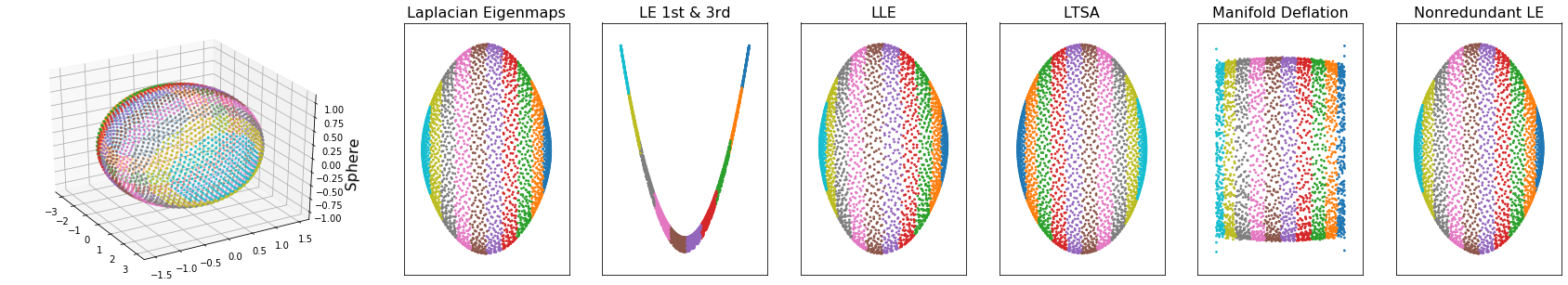}
    \caption{Embeddings of the sphere which show that only Manifold Deflation generates any form of meaningfully different embedding.}
    \label{fig:sphere embeddings}
\end{figure}

\end{document}